\documentclass{article}

\usepackage{microtype}
\usepackage{graphicx}
\usepackage{subcaption}
\usepackage{booktabs} 

\usepackage{hyperref}

\newcommand{\pdfeq}[1]{%
    \begin{equation*}
        \begin{gathered}
        \includegraphics[scale=0.8]{figures/#1.pdf}
        \vspace*{-0.5em}
        \end{gathered}
        \label{eq:#1}
    \end{equation*}
}

\usepackage[accepted]{icml2026}

\usepackage{amsmath}
\usepackage{amssymb}
\usepackage{mathtools}
\usepackage{amsthm}
\usepackage{array}      
\usepackage{tabularx}   
\usepackage{booktabs}   

\usepackage[capitalize,noabbrev]{cleveref}

\theoremstyle{plain}
\newtheorem{theorem}{Theorem}[section]
\newtheorem{proposition}[theorem]{Proposition}

\newtheorem{corollary}[theorem]{Corollary}
\theoremstyle{definition}
\newtheorem{definition}[theorem]{Definition}

\newtheorem{conjecture}[theorem]{Conjecture}
\theoremstyle{remark}

\usepackage{xcolor}
\definecolor{orange}{HTML}{EC9940}
\definecolor{red}{HTML}{C35250}
\definecolor{green}{HTML}{98C24C}
\definecolor{purple}{HTML}{A573D3}
\definecolor{blue}{HTML}{78C3FB}

\usepackage{bm}

\newcommand\syn[1]{\bm{\textcolor{orange}{#1}}}
\newcommand\sem[1]{\bm{\textcolor{red}{#1}}}
\newcommand\expl[1]{\bm{\textcolor{green}{#1}}}

\usepackage{enumitem}
\setlist[itemize]{leftmargin=*}
\setlist[enumerate]{leftmargin=*}
\setlist{noitemsep}

\usepackage[disable,textsize=tiny]{todonotes}

\icmltitlerunning{From Mechanistic to Compositional Interpretability}

\begin{document}
\setlength{\baselineskip}{0.98\baselineskip}
\setlength{\parskip}{0.9\parskip}

\twocolumn[
  \icmltitle{From Mechanistic to Compositional Interpretability}



  \icmlsetsymbol{equal}{*}

  \begin{icmlauthorlist}
    \icmlauthor{Ward Gauderis}{equal,ailab}
    \icmlauthor{Thomas Dooms}{equal,independent}
    \icmlauthor{Steven T. Homer}{independent}
    \icmlauthor{Kola Ayonrinde}{aisi}
    \icmlauthor{Geraint A. Wiggins}{ailab}
  \end{icmlauthorlist}

  \icmlaffiliation{ailab}{Vrije Universiteit Brussel, Brussels, Belgium}
  \icmlaffiliation{aisi}{UK AI Security Institute, London, United Kingdom}
  \icmlaffiliation{independent}{Independent}

  \icmlcorrespondingauthor{Ward Gauderis}{ward.gauderis@vub.be}
  \icmlcorrespondingauthor{Thomas Dooms}{doomsthomas@gmail.com}

  \icmlkeywords{Compositionality, Category Theory, Interpretability, Complexity}

  \vskip 0.3in
]



\printAffiliationsAndNotice{\icmlEqualContribution}

\begin{abstract}
Mechanistic interpretability aims to explain neural model behaviour by reverse-engineering learned computational structure into human-understandable components.
Without a formal framework, however, mechanistic explanations cannot be objectively verified, compared, or composed.
We introduce \textit{compositional interpretability}, a category-theoretic framework grounded in the principles of compositionality and minimum description length.
\textit{Compositional interpretations} are pairs of syntactic and semantic mappings that must commute to enforce consistency between a model's decomposition and its observed behaviour.
We deconstruct explanation quality into measures of \textit{faithfulness} and \textit{complexity} to cast interpretability as a constrained optimisation problem, and introduce \textit{compressive refinement} to systematically restructure models into simpler parts without altering their function.
Finally, we derive a parsimony criterion under which syntactic compression theoretically guarantees more concise, human-aligned explanations.
Our framework situates prominent mechanistic methods as subclasses of refinement, and clarifies why their compressibility heuristics tend to align with human interpretability.
Our work provides a measurable, optimisable blueprint for automating the discovery and evaluation of mechanistic explanations.
\end{abstract}

\section{Introduction} \label{sec:introduction}

Neural models derive their capabilities from emergent complexity rather than explicit design. While they learn sophisticated and meaningful functions, their behaviour is difficult to understand because their internal representations are opaque to humans \cite{ayonrindeMathematicalPhilosophyExplanations2025}. Yet interpreting them could yield new scientific insights \citep{evo2, Nainani2025, Pearce2026.04.10.717844} and give us greater control over their behaviour \cite{templeton2024scaling}. Doing so requires a measurable definition of interpretation quality that can be tested and optimised \cite{barbieroFoundationsInterpretableModels2025, madsenInterpretabilityNeedsNew2024}. Without it, the resulting explanations cannot systematically be compared or composed.

Interpretability approaches can be seen as spanning a spectrum from intrinsically interpretable models to post-hoc explanations. Post-hoc methods such as saliency and attribution \citep{binder2016layerwiserelevancepropagationneural} explain individual predictions through surface-level heuristics, but provide little insight into the model's internal mechanisms \citep{ayonrindeMathematicalPhilosophyExplanations2025, adebayo2020sanitycheckssaliencymaps, kim2018interpretabilityfeatureattributionquantitative}. Conversely, intrinsically interpretable models aim to make the structure of the computation transparent by design but require changes to architecture or training \citep{madsenInterpretabilityNeedsNew2024}. Mechanistic interpretability sits between these endpoints by seeking to relate observed behaviour to internal features \citep{alain2018understandingintermediatelayersusing} or to mechanisms discovered through compressibility heuristics such as sparsity \citep{ayonrindeInterpretabilityCompressionReconsidering2024, bricken2023monosemanticity, gao2024scalingevaluatingsparseautoencoders}. While the goal is sound, current mechanistic methods approximate faithfulness rather than enforce it structurally.

The standard reverse-engineering view constructs interpretations in three stages: decomposition, description, and validation \cite{sharkeyOpenProblemsMechanistic2025}. These correspond to breaking a model into simpler functional parts, assigning them explanatory hypotheses, and testing their correctness, respectively. Decomposition fixes the structural foundation, or syntax, on which the latter two stages operate. If that basis is non-compositional, then description and validation can be locally correct while remaining globally uninformative. With a flawed syntax, even perfectly described parts will not add up to a complete explanation of global behaviour.

\begin{figure*}[t]
    \centering
    \includegraphics[width=0.65\textwidth]{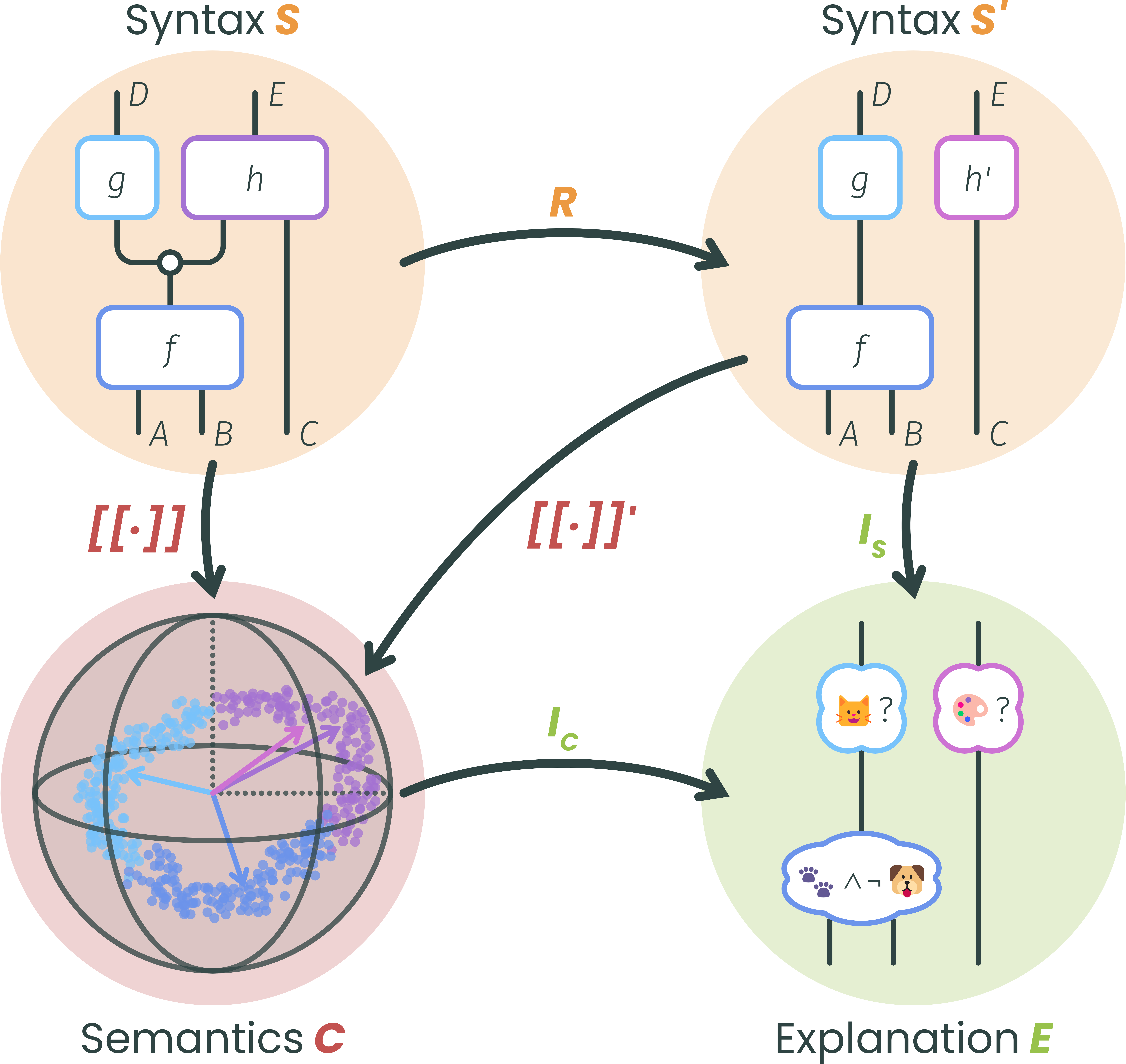}
    \setlength{\belowcaptionskip}{-1em}
    \caption{
    A commutative diagram illustrating compositional interpretability through compressive refinement for a model that classifies animals and their colour. In the original decomposition, a string diagram in $\syn S$, the mechanisms appear structurally entangled even though their learned representations $\sem{[[\,\cdot\,]]}$ in $\sem C$ are not.    
    Through a compressive refinement $\syn R$, a new model decomposition $\syn {S'}$ is discovered that clearly separates the two mechanisms, while the new implementation $\sem {[[\,\cdot\,]]'}$ remains faithful to the model's input-output behaviour in $\sem{C}$.
    Based on this new syntax, a compositional interpretation $\expl{I_S}$ can assign much simpler explanations in $E$ to the disentangled mechanisms. Because all paths commute, the resulting explanation coincides with the post-hoc observed behaviour $\sem{I_C}$ of the original model.
    }
    \label{fig:diagram}
\end{figure*}

To avoid the above shortcomings, we argue that interpretations should be \textbf{compositional}, made of individually describable parts whose composition equals the whole; \textbf{faithful}, approximating the model's internal computation rather than merely its input-output behaviour; and \textbf{concise}, having low description length for a human interpreter. Compositionality enables bidirectional tracing between local parts and global behaviour, while faithfulness prevents explanations from becoming just-so stories. Together, these criteria cast interpretability as constrained optimisation: among all faithful compositional explanations, prefer the most concise.

This paper formalises interpretability through the notion of a \textit{compositional interpretation} (\cref{sec:compositional}), which requires grounding a model in both its structure (syntax) and behaviour (semantics). We characterise the quality of such an interpretation by its faithfulness and complexity (\cref{sec:complexity}), and upper-bound complexity by a tractable measure that can be directly optimised. The optimisation takes the form of \textit{compressive refinement} (\cref{sec:refinement}), restructuring the model into simpler parts while preserving its semantics. 
We show how compressive refinement aligns interpretability under the principles of explanatory optimism and compositional sparsity.
Finally, we situate compositional interpretability relative to other prominent formalisms (\cref{sec:related_work}). A summary of the formal notation and colour-coding is provided in \cref{app:notation}.

\section{Defining Compositional Interpretations} \label{sec:compositional}

The objective of interpreting neural models is rarely specified, leading to ambiguity about what constitutes a valid explanation.
We build on an \textit{explanatory view} of interpretability, which holds that explaining model behaviour requires examining internal mechanisms, not just input-output pairs. This perspective distinguishes between two forms of faithfulness \cite{ayonrindeMathematicalPhilosophyExplanations2025}. 
\textit{Behavioural faithfulness} means that the explanation and the model produce the same outputs. Post-hoc interpretability methods typically aim for this, describing what a model computes by generating explanations that correlate with its data-local surface behaviour, thereby answering questions like ``Why did the model make this particular decision?''
In contrast, \textit{explanatory faithfulness} means the internal steps of the step-by-step explanation match the model's own mechanisms, not just its input-output behaviour. Mechanistic interpretations pursue this by examining how a model operates and elucidating the mechanisms by which it reaches conclusions. The mechanistic approach addresses broader questions such as ``How did the model solve this general class of problems?'', aiming to explain the model's generalisation capabilities rather than just individual decisions \cite{sharkeyOpenProblemsMechanistic2025}.

\begin{definition}[Explanatory faithfulness]
    \label{def:explanatory_faithfulness}
    An explanation $E$ is \textit{explanatorily faithful} to a model $M$ over a data distribution $\mathcal{D}$, to the extent that intermediate activations $s_i$ at each layer $i$ that are given by an algorithmic explanation $E$ closely match the intermediate activations $x_i$ of the model $M$ for input data in $\mathcal{D}$ \cite{ayonrindeMathematicalPhilosophyExplanations2025}.
\end{definition}
Two aspects deserve emphasis. First, explanation quality depends on the interpreter's goals and abilities. Second, faithfulness is relative to a data distribution and may be approximate. This motivates a non-binary notion of faithfulness, which we develop in \cref{sec:refinement}. For now, we discuss the structural implications of explanatory faithfulness.

\paragraph{Compositionality} Neural networks expose a limitation of naive reductionism: despite complete knowledge of their low-level components, their emergent behaviour remains opaque, posing a fundamental challenge for interpretability \cite{ayonrindeMathematicalPhilosophyExplanations2025}. Achieving explanatory faithfulness---matching an explanation's structure to a model's computational mechanisms---requires decomposing the model into meaningful components. While a model's layers offer a natural starting point for such a decomposition (\cref{def:explanatory_faithfulness}), there is no guarantee that these architectural boundaries align with intelligible atoms of computation. 
Indeed, because multi-layer computations can be flattened or vice versa \cite{linWhyDoesDeep2017,danhoferPositionTheoryDeep2025}, architectural depth is an implementation choice. Coupled with the discovery of cross-layer circuits \cite{olsson2022incontextlearninginductionheads, lindsey2024sparse-crosscoders}, this suggests the fundamental units of computation are functional, not architectural.
These observations motivate the need for a more general, functionally grounded principle for explanatory faithfulness: \textit{compositionality}, commonly attributed to Frege \cite{coeckeCompositionalityWeSee2021}.
\begin{definition}[Compositionality]
The meaning of the whole can be derived from the meanings of its parts and how these parts are structured together.
\end{definition}
Compositionality describes the relationship between a system's global properties and those of its components \cite{pucaObstructionsCompositionality2023, https://doi.org/10.1111/j.1755-2567.1970.tb00434.x}. For interpretability, this enables a divide-and-conquer approach: a model's behaviour should emerge from the local understanding of its parts (composition), while behaviour must also be traceable back to the contributions of those parts (decomposition) \cite{pucaObstructionsCompositionality2023}.
However, compositionality is not theory-neutral; it requires a formal theory of both syntax (the parts and their structure) and semantics (their meaning) \cite{hupkesCompositionalityDecomposedHow2020}. The central challenge, therefore, is to identify a formal syntax and semantics for a model that is both faithful to its mechanisms and useful for human understanding.

Category theory provides the mathematical language to formalise this distinction between syntax and semantics and to reason about compositional structures \cite{fongSevenSketchesCompositionality2018,dudzikCategoriesAI2022}. Compositional methods in AI are largely inspired by fields like quantum mechanics, which shift the focus of study from isolated components to their structured interactions \cite{coeckePicturingQuantumProcesses2018}. 
These methods combine explicit structure with distributed representations, helping overcome the curse of dimensionality in model generalisation and interpretation \cite{tullCompositionalInterpretabilityXAI2024,coeckeCompositionalityWeSee2021,hitzlerNeuroSymbolicArtificialIntelligence2021}.
The application of category theory in AI has already produced unifying frameworks for organising insights across the field \cite{shieblerCategoryTheoryMachine2021,dudzikCategoriesAI2022}, whose utility has been demonstrated in model design \cite{gavranovicCategoricalDeepLearning2024, khatriAnatomyAttention2024, rodatzPatternLanguageMachine2025} and conceptual, causal and generative modelling \cite{tullConceptualSpacesQuantum2023,gauderisQuantumTheoryKnowledge2023, lorenzCausalModelsString2023, tullActiveInferenceString2023}.

\paragraph{Compositional interpretability} \textit{String diagrams} \cite{coeckePicturingQuantumProcesses2018} provide the syntactic skeleton needed to separate a model's wiring from its computational behaviour. They describe the flow of information as pure structure (see \cref{fig:diagram} and \cref{app:string_diagram}). By endowing this bare syntax with a semantic representation, we form a \textit{compositional model} \cite{tullCompositionalInterpretabilityXAI2024}.
\begin{definition}[Compositional model] \label{def:compositional_model}
A \textit{compositional model} $M = (\syn D, \syn S, \sem C, \sem{[[\,\cdot\,]]})$ consists of:
\begin{itemize}
\item a string diagram $\syn D: \syn I \to \syn O$ in the monoidal category $\syn S$, representing the \textit{syntax} $\syn S$ of $M$;
\item a monoidal category $\sem C$, providing the \textit{semantics} of $M$;
\item a monoidal \textit{representation} functor $\sem{[[\,\cdot\,]]}: \syn S \to \sem C$, which assigns a semantic instantiation in $\sem C$ to each syntactic component in $\syn S$, preserving the diagrammatic structure.
\end{itemize}
The diagram $\syn D$ is built from a finite collection of components. Its \textit{signature} $\syn{S_D}$ consists of the set of all individual wire types (objects $\text{ob}(\syn{S_D})$) and boxes (generators $\text{mor}(\syn{S_D})$) occurring in \mbox{$\syn D$\footnotemark}.
\footnotetext{
The syntactic category $\syn S = \text{Free}(\syn{S_D})$ is the free monoidal category generated by $\syn{S_D}$, comprising all valid diagrams formed from $\text{ob}(\syn{S_D}$) and $\text{mor}(\syn{S_D})$.
The signature $\syn{S_D}$ may also admit equations $\text{eq}(\syn{S_D})$ specifying equivalence between diagrams.
} 
The image of $\syn{S_D}$ under $\sem{[[\,\cdot\,]]}$ is denoted as $\sem{C_{D}}$, the minimal set of semantic elements required to implement the representation of the diagram $\syn D$.
\end{definition}

Within this framework, the \textit{syntax} $\syn S$, described by the string diagram $\syn D$, defines the model's abstract internal structure, such as neurons, layers, or attention heads in neural models. The diagram provides a high-level structure that can be inspected or reasoned about. 
The \textit{semantics} $\sem C$ provides the model's concrete instantiation or implementation, enabling practical computation and evaluation. The representation $\sem{[[\,\cdot\,]]}: \syn S \to \sem C$ links syntax to semantics, assigning (learned) functions, distributions, or matrices to these syntactic components  (\cref{fig:diagram}).
For interpretability, this distinction between syntax and semantics offers a useful lens for decomposing model behaviour into understandable parts. It elucidates how global model behaviour arises from learned semantic representations composed by syntax. Both features and mechanisms \cite{olahZoomIntroductionCircuits2020} can be viewed as syntactic components, differing primarily in the degree to which they can further be decomposed into subdiagrams. We can understand how different syntactic constructions are (approximately, within known bounds) equivalent by relating them to their concrete semantic instantiations.

While any model can be framed compositionally, this is only useful when oriented toward a specific goal \cite{duneauComparativeFrameworkCompositional2025}. To this end, we identify \textit{compositional interpretations} \cite{tullCompositionalInterpretabilityXAI2024}: human-understandable explanations tied to both a model's syntax and semantics.
%
\begin{definition}[Compositional Interpretation] \label{def:compositional_interpretation}
Given a compositional model $M = (\syn D, \syn S, \sem C, \sem{[[\,\cdot\,]]})$ and a signature of human-interpretable terms $\expl H$, a \textit{compositional interpretation} consists of two compatible mappings:
\begin{itemize}
    \item A syntax-based, \textit{compositional interpretation} $\expl{I_S}: \syn S \to \expl H$, which provides \textit{syntactic grounding} by assigning explanations to the structural components of $M$.
    \item A semantics-based, \textit{post-hoc interpretation} $\expl{I_C}: \sem C \to \expl H$, which provides \textit{semantic grounding} by producing explanations based on the behaviour implemented by $M$.
\end{itemize}
For consistent groundings, this diagram must commute\footnote{Categorically, both mappings would ideally be complete monoidal functors. In practice, these constraints are relaxed to partial maps to account for (as yet) unexplained components.}:
\pdfeq{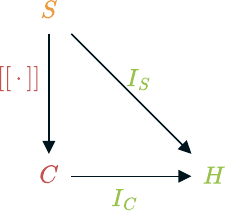}
\end{definition}
The commutativity condition, $\expl{I_S} = \expl{I_C} \circ \sem{[[\,\cdot\,]]}$, is the cornerstone of this definition. It enforces consistency between two forms of interpretation (\cref{fig:diagram}).
The compositional interpretation $\expl{I_S}$ ties the explanation to the model's structure, while the post-hoc interpretation $\expl{I_C}$ ties it to the model's input-output behaviour.
Without this condition, interpretations risk becoming disconnected from the mechanisms they claim to describe.
An interpretation defined solely in terms of semantics, without factoring through a valid syntax, remains causally disconnected from the model's mechanisms: it describes {\it what} the model does, not {\it how} it is done. Consequently, components with distinct implementations might be given identical explanations despite structural differences.
Conversely, an interpretation that is defined solely in terms of syntax labels internal components without verification against their semantics: these labels may not reflect their actual function.
Components with identical semantics might be given different, misleading explanations based on superficial structural distinctions.
By requiring syntax and semantics to agree, a compositional interpretation connects mechanism to behaviour.

In practice, interpretations are partial, meaning they may not account for every component of a model. A \textit{compositionally-interpretable model} admits a complete compositional interpretation $\expl{I_S}$, where every syntactic component receives a human-understandable explanation.
\citet{tullCompositionalInterpretabilityXAI2024} show that this definition generalises the existing notion of intrinsically interpretable models, such as decision trees or linear models, which are considered transparent precisely because their explicit structure permits direct and complete interpretation . The central objective of compositional interpretability is therefore 
{\it the search for compositional interpretations that are as complete as possible}.

\section{Evaluating Compositional Interpretations} \label{sec:complexity}

Compositionality is only a \textit{structural constraint}: it defines the space of valid explanations but does not distinguish the good from the bad. When offered multiple valid choices, which explanation should be preferred?
We take inspiration from Kowalski's (\citeyear{kowalskiAlgorithmLogicControl1979}) dictum 
$$\text{Algorithm} = \sem{\textbf{Logic}} + \syn{\textbf{Control}}.$$ 
In logic programming, the logic component determines the meaning of the algorithm, whereas the control component only affects its implementation. By separating the two, Kowalski argued that computer programs can be improved by freely optimising control for computational efficiency, while the algorithm's logic must be held constant.
Compositional interpretability mirrors this separation as\footnote{This also parallels Marr's \yrcite{marr} computational and algorithmic levels of analysis, respectively.}
$$\text{Model} = \sem{\textbf{Semantics}} + \syn{\textbf{Syntax}}.$$
We treat the semantics as the fixed logic---the input-output behaviour we aim to understand---and the syntax as the variable control structure that decomposes it. Our objective differs from Kowalski's only in the target metric: rather than optimise syntax for efficiency, we optimise for intelligibility.

Our constrained optimisation view of interpretability is formalised through the \textit{Minimum Description Length} (MDL) principle \cite{grunwald2007minimum}. In model selection, MDL identifies the best model ${M}$ for a dataset $\mathcal{D}$ as the one that minimises the total description length $L(\,\cdot\,)$ in bits:
$$L(\mathcal{D}) = L(M) + L(\mathcal{D} \mid M).$$
The total description cost is the length of the model $L(M)$ plus the length of the data encoded, given that model $L(\mathcal{D} \mid M)$.
This two-part code operationalises Occam's razor: all else being equal, prefer the more concise explanation \cite{grunwald2007minimum}.
We apply this principle at a meta-level: interpretability is framed as a constrained communication problem where the goal is to communicate the model's semantics $\sem{[[D]]}$ as efficiently as possible in a human-comprehensible language $\expl H$ \cite{ayonrinde2024interpretabilitycompressionreconsideringsae, braun2025interpretabilityparameterspaceminimizing}. We treat the model's fixed semantics $\sem{[[D]]}$ as the data and the compositional interpretation $\expl{I_S}(\syn D)$ as the hypothesis. The description length of the semantics then decomposes as:
$$L(\sem{[[D]]}) = \underbrace{L(\expl{I_S}(\syn D))}_{\text{Complexity}} + \underbrace{L(\sem{[[D]]} \mid \expl{I_S}(\syn D))}_{\text{Faithfulness}}$$
The first term measures the \textit{complexity} of the explanation. The second quantifies \textit{faithfulness} in the information lost by the explanation. It is the cost of correcting the explanation to match the model's behaviour.
Formulating interpretability through MDL presents distinct challenges compared to standard machine learning: the hypothesis class potentially spans all human-interpretable statements $\expl H$, making the complexity of an explanation subjective and dependent on the interpreter's goals, priors, and computational capabilities. However, unlike inferring a model from noisy empirical data, we have access to the system's closed-form description. Knowing the ground truth without aleatoric uncertainty allows us to separate the MDL terms. Rather than minimising the sum as a soft trade-off, we treat it as a rate-distortion problem and minimise the explanation's complexity subject to a strict lower bound on faithfulness.

We thus refine the central criterion of compositional interpretability as \textit{the search for complete compositional interpretations that yield the least {complex} explanations while remaining {faithful} up to a user-specified threshold}. 
In the remainder of this section, we formalise these measures for compositional interpretations so their optimisation can be automated. Evaluating them naively in the human-interpretable category $\expl H$ would require subjective human involvement. To bypass this, we instead ground them directly in the model's syntax $\syn S$ and semantics $\sem C$.

\paragraph{Faithfulness} The \textit{faithfulness measure} $L(\sem{[[D]]}\mid \expl{I_S}(\syn D))$ quantifies the information needed to correct an explanation. Rate-distortion theory operationalises this by bounding a semantic distortion metric $d$ (e.g., $L_2$ norm or KL-divergence) between the model's true semantics and those reconstructed from the explanation. However, restricting this bound to the global input-output behaviour $\sem{[[D]]}$ ensures only behavioural faithfulness, lacking the mechanistic insight required by \cref{def:explanatory_faithfulness}. While compositional interpretations are structurally well-defined (\cref{def:compositional_interpretation}), they are not inherently informative: describing a component merely as ``a mathematical function'' is valid but trivial. To ensure explanations capture specific structural mechanisms, we impose the stricter measure of \textit{compositional faithfulness}: the interpretation must be detailed enough to reconstruct the semantics of {\it each} component that it explains.
\begin{definition}[Compositional faithfulness]
\label{def:compositional_faithfulness}
Given a compositional model $M = (\syn D, \syn S, \sem C, \sem{[[\,\cdot\,]]})$, a compositional interpretation $\expl{I_S}: \syn S \to \expl H$ is \textit{compositionally faithful} if there exists a \textit{reconstruction} $\sem{I^*_C}: \expl H \to \sem C$ such that the following diagram commutes:
\pdfeq{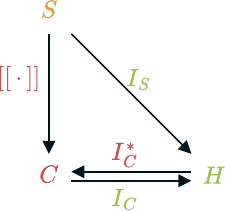}
This constraint is relaxed up to a tolerance $\epsilon \in \mathbb{R}$ to reconcile faithfulness and complexity\footnote{Categorically, relaxing strict inversion to an adjunction $\sem{I^*_C} \dashv \expl{I_C}$ formalises lossy abstraction: $\expl{I_C}$ discards details while $\sem{I^*_C}$ reconstructs the canonical implementation consistent with the explanation. However, our analysis does not rely on such assumptions.}.
Given a metric $d$ on boxes (parallel morphisms) in $\sem C$, we require that $\forall \syn f \in \text{mor}(\syn S)$: $$d(\sem{[[\,f\,]]}, \sem{I^*_C}(\expl{I_S}(\syn f))) \le \epsilon$$
\end{definition}
This condition subsumes explanatory and behavioural faithfulness: reconstructing individual component mechanisms necessarily recovers global behaviour upon composition. Consequently, the interpreter can understand the global system by combining local component explanations according to the diagrammatic structure. 
The metric $d$ defines the semantic reconstruction error relevant to the interpreter, making faithfulness relative rather than absolute. By relaxing $d$ to a premetric over specific subdistributions, we ensure faithfulness where it matters while permitting abstraction elsewhere. For example, evaluating $d$ only over naturally occurring or task-relevant data captures observational alignment, whereas evaluating it over interventional data ensures the explanations are causally faithful \cite{joshiCausalityKeyInterpretability2026}.

\paragraph{Complexity} The \textit{complexity measure} $L(\expl{I_S}(\syn D))$ of an explanation is inherently more elusive than faithfulness since it depends on the interpreter's subjective frame of reference. The lack of a universal metric has posed a significant obstacle to automated evaluation of interpretations. However, the dual nature of compositional interpretations offers a tractable surrogate. Just as we operationalised faithfulness in the semantics $\sem C$ instead of $\expl H$, we bound complexity by the structural derivation of the explanation rather than by the explanation itself. Relying on the factorisation $\expl{I_S} = \expl{I_C}\circ \sem{[[\,\cdot\,]]}$ and functoriality of $\sem{[[\,\cdot\,]]}$, we define the \textit{compositional description length}: the cost of describing the model's compositional representation plus the complexity of interpreting its components' semantics.
\begin{definition}[Compositional description length] \label{def:compositional_description_length}
Let $M = (\syn D, \syn S, \sem C, \sem{[[\,\cdot\,]]})$ be a compositional model and $\expl{I_S} : \syn S \to \expl H$ a compositional interpretation for $M$. The \textit{compositional description length} $L(M, \expl{I_S})$ of its explanation $\expl{I_S}(\syn D)$ is defined as the sum of its representation and interpretation complexities.
$$
L(M, \expl{I_S}) := L^{\sem{\text{rep}}}(M) + L^{\expl{\text{int}}}(M, \expl{I_C})
$$
Where \textit{representation complexity} $L^{\sem{\text{rep}}}(M)$ is the syntax-based cost of describing the semantics of the individual components in $\syn{S_D}$ and their wiring in $\syn D$ separately. $$L^{\sem{\text{rep}}}(M) := L(\syn D) + \sum_{\syn f \in \text{mor}(\syn{S_D})} L(\sem{[[\,f\,]]})$$
And \textit{interpretation complexity} $L^{\expl{\text{int}}}(M, \expl{I_C})$ is the cost of describing the post-hoc interpretation method for the individual semantic elements in $\sem{C_D}$. $$L^{\expl{\text{int}}}(M, \expl{I_C}) := L(\expl{I_C}\rvert_{\sem{{C}_{D}}})$$
\end{definition}
The decomposition splits the description length of an explanation into two distinct parts: an objective structural cost and a subjective alignment cost. 
The \textit{representation complexity} $L^{\sem{\text{rep}}}(M)$ depends solely on the model. It captures the trade-off between the granularity of the decomposition $\syn D$ and the semantic complexity of the atomic units $\sem{[[f]]}$. This makes it a tractable target for automated optimisation (\cref{app:worked_example}).
Conversely, the \textit{interpretation complexity} $L^{\expl{\text{int}}}(M, \expl{I_C})$ isolates the subjective difficulty of aligning the semantics of the atomic units in $\syn D$ with human concepts. This cost is inherently variable, scaling with the conceptual gap \cite{ayonrindePositionInterpretabilityBidirectional2025} between the components' semantics relative to the interpreter's mental model \cite{sutterNonLinearRepresentationDilemma2025}.
Optimising $L(M,\expl{I_S})$ therefore balances syntactic parsimony with semantic simplicity: we seek a decomposition that is concisely described yet composed of parts that are naturally aligned with the interpreter. 
As $L(M, \expl{I_S})$ represents the information cost of the compositional derivation, it upper-bounds the theoretical minimal description length $L(\expl{I_S}(\syn D))$ of the explanation itself.
\begin{theorem}[Description length bound]
\label{prop:description_length_bound}
Compositional description length $L(M, \expl{I_S})$ is an upper bound on the minimal description length $L(\expl{I_S}(\syn D))$ of a model's explanation.
$$
L(\expl{I_S}(\syn D))\le L(M, \expl{I_S})
$$
The proof is deferred to \cref{app:proof}.
\end{theorem}

\section{Improving Compositional Interpretations} \label{sec:refinement}

Since $L(M,\expl{I_S})$ corresponds to the length of a realisable encoding, it serves as a tractable surrogate to optimise $L(\expl{I_S}(\syn D))$. Optimising this bound formalises the natural workflow in mechanistic interpretability: identifying a suitable decomposition, interpreting its components, and verifying their descriptions causally \cite{sharkeyOpenProblemsMechanistic2025}. Since structurally modifying a model is more tractable than solving the subjective alignment problem, recent research has predominantly been successful in reducing representation complexity $L^{\sem{\text{rep}}}(M)$ \cite{gaoWeightsparseTransformersHave2025, dunefskyTranscodersFindInterpretable2024, bushnaqStochasticParameterDecomposition2025}, while interpretation complexity $L^{\expl{\text{int}}}(M, \expl{I_C})$ remains less explored \cite{sutterNonLinearRepresentationDilemma2025,voitaInformationTheoreticProbingMinimum2020,oldfieldLinearProbesDynamic2025}. In this work, we align with this strategy, focusing on the first step: finding decompositions that minimise representational cost. 
We formalise this process as \textit{syntactic refinement} \cite{tullCompositionalInterpretabilityXAI2024}: compressing the model's representation by optimising the syntax while keeping the semantics fixed.

\begin{definition}[Syntactic refinement]
\label{def:syntactic_refinement}
A \textit{syntactic refinement} of a compositional model $M = (\syn D, \syn S, \sem C, \sem{[[\,\cdot\,]]})$ into a semantically equivalent model $M' = (\syn{D'}, \syn{S'}, \sem C, \sem{[[\,\cdot\,]]'})$ is a functor $\syn R: \syn S \to \syn{S'}$ such that $\syn{D'} = \syn R(\syn D)$\footnote{The dual functor $\syn{R'}: \syn{S'} \to \syn S$ yields \textit{syntactic abstraction}. By functor composition, one can selectively refine and abstract representations while preserving a trace to the original model.}.

For unchanged semantics $\sem C$, this diagram must commute\footnote{
In practice, this constraint is relaxed up to a tolerance $\epsilon \in \mathbb{R}$ to reconcile (global) faithfulness and description length. Given a metric $d$ (or generalisation thereof) on $\text{Hom}_C(\sem{[[I]]}, \sem{[[O]]})$, we require $d(\sem{[[\,D\,]]},\sem{[[\,D'\,]]'}) \le \epsilon$.}:
\pdfeq{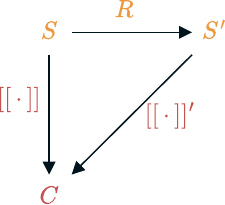}
If $L^{\sem{\text{rep}}}(M') \le L^{\sem{\text{rep}}}(M)$, the refinement reduces the representation complexity and is \textit{compressive}.
\end{definition}
Compressive refinement $\syn R$ operationalises the redistribution of information from semantic atoms to syntactic structure, subject to $\sem{[[\,\cdot\,]]} = \sem{[[\,\cdot\,]]'} \circ \syn R$. By unpacking a black box with high $L(\sem{[[D]]})$ into a diagram of simpler sub-components (increased $L(\syn D)$ but lower $\sum_{\syn f \in \text{mor}(\syn{S_D})} L(\sem{[[f]]})$), the model's internal regularities become explicit (\cref{fig:diagram}). An efficient refinement seeks the optimal granularity where the reduction in component complexity outweighs the cost of describing the additional wiring (\cref{app:worked_example}).

This gives a formal rationale for the search for ``atoms of computation'' \cite{olahZoomIntroductionCircuits2020}: primitive components that are algebraically simple and structurally explicit. However, mathematical simplicity does not guarantee human understanding. As captured by the interpretation complexity term in \cref{def:compositional_description_length}, syntactic compression alone is insufficient. A refinement might simplify syntax at the cost of introducing components that are harder to interpret. For instance, a syntactic refinement could decrease representation complexity $L^{\sem{\text{rep}}}(M)$ by replacing linearly alignable components with nonlinear ones \cite{elhage2022toymodelssuperposition}, thereby increasing interpretation complexity ${L}^{\expl{\text{int}}}(M, \expl{I_C})$. Consequently, a compressive refinement improves the explanation only if the reduction in representation complexity outweighs any increase in interpretation complexity. We formalise this necessary condition via the \textit{parsimony criterion}, a decision rule that follows from \cref{prop:description_length_bound}.
\begin{proposition}[Parsimony criterion for refinement]
\label{the:parsimony_criterion}
Let $M'$ be a compressive refinement of $M$. Let $\expl{I_S}$ and $\expl{I_S'}$ be the compositional interpretations of $M$ and $M'$ respectively, with the same semantic grounding $\expl{I_C}: \sem C \to \expl H$.
The refined model $M'$ yields a more parsimonious explanation in terms of compositional description length 
$$L(M', \expl{I_S'}) \le L(M, \expl{I_S})$$
if and only if the gained representation simplicity outweighs the interpretation complexity:
$$\underbrace{L^{\sem{\text{rep}}}(M) - L^{\sem{\text{rep}}}(M')}_{\text{Decrease in representation complexity}} \ge \underbrace{L^{\expl{\text{int}}}(M', \expl{I_C}) - L^{\expl{\text{int}}}(M, \expl{I_C})}_{\text{Increase in interpretation complexity}}$$
\end{proposition}%

The parsimony criterion underscores that a compressive refinement is not an end in itself but a means to balance two tensions in finding concise compositional explanations:
\begin{enumerate}
    \item \textit{Representation vs.\ interpretation complexity:} There is an external risk that optimising the objective representation complexity $L^{\sem{\text{rep}}}(M)$ comes at the cost of the subjective interpretation complexity ${L}^{\expl{\text{int}}}(M, \expl{I_C})$. A concise description is not necessarily easy for a human to interpret.
    \item \textit{Wiring vs.\ component complexity:} Internal to representation complexity, there is a tension between the wiring complexity $L(\syn D)$ and the semantic complexity of the individual components $\sum_{\syn f \in \text{mor}(\syn{S_D})} L(\sem{[[f]]})$. Decomposing components reduces their individual cost but increases the diagram size $\syn D$. An exponentially large diagram undermines behavioural faithfulness, as the interpreter cannot tractably compose all interactions.
\end{enumerate}
While these trade-offs suggest that naively optimising $L^{\sem{\text{rep}}}(M)$ alone could potentially lead to uninterpretable model decompositions, practical evidence indicates that compression often yields interpretable results \cite{gaoWeightsparseTransformersHave2025, dunefskyTranscodersFindInterpretable2024, braun2025interpretabilityparameterspaceminimizing, autoencoders}. Under compressive refinements, the two pressures described above are often aligned rather than adversarial. We argue this can be made deliberate, addressing the first tension through a comonotonic coding distribution and the second through compositional sparsity.

\paragraph{Representation vs.\ interpretation complexity} To resolve the first tension between representation and interpretation complexity, we draw on the intuition that the regularities a model exploits to solve a task often correspond to concepts humans use to understand that task. This intuition is captured by the \textit{Weak Principle of Explanatory Optimism}, a conjecture usually left implicit in interpretability \cite{ayonrindeMathematicalPhilosophyExplanations2025}. 
\begin{conjecture}[Weak principle of explanatory optimism]
\label{con:optimism}
Most important behaviour of an intelligence with implicit explanatory knowledge within some explanatory complexity class is human-understandable \cite{ayonrindeMathematicalPhilosophyExplanations2025}.
\end{conjecture}
Since structure is fundamentally linked to compressibility, compressive refinement acts as a filter for these regularities. Lossy compression that retains only the most important model behaviour (measured by $d$ in \cref{def:syntactic_refinement}) strips away redundancy to isolate the core mechanisms. The conjecture implies that if the model performs well, these isolated mechanisms are human-understandable.
To ensure alignment to human concepts, we encode explanatory optimism in the coding distribution that defines representation complexity, $L^{\sem{\text{rep}}}(M) := - \log_2 P^{\sem{\text{rep}}}(M)$ \cite{grunwald2007minimum}. By choosing a $P^{\sem{\text{rep}}}(M)$ that encodes our prior \textit{belief} in explanatory optimism, assigning higher probability to representations we expect to be interpretable (such as those exhibiting sparsity, low dimensionality, or hierarchy), $L^{\sem{\text{rep}}}(M)$ measures how well a model satisfies this prior, aligning representation complexity with interpretation complexity.
\begin{corollary}[Comonotonic coding]
\label{the:comonotonicity}
Any compressive refinement satisfies the parsimony criterion (\cref{the:parsimony_criterion}) if its coding distribution $P^{\sem{\text{rep}}}(M)$ is \textnormal{comonotonic} with the induced interpretation distribution:
$$P^{\sem{\text{rep}}}(M) \le P^{\sem{\text{rep}}}(M') \implies 2^{-L^{\expl{\text{int}}}(M, \expl{I_C})} \le 2^{-L^{\expl{\text{int}}}(M', \expl{I_C})}$$
\end{corollary}
Consequently, the closer our chosen coding distribution captures actual interpretation complexity, the more reliably a compressive refinement improves the full compositional description length\footnote{Readers reluctant to assert explanatory optimism as a literal belief can instead adopt a \textit{luckiness function} within the Normalised Maximum Likelihood framework \cite{grunwaldMinimumDescriptionLength2019}. This replaces the subjective Bayesian \textit{belief} with a pragmatic \textit{hope} for interpretability; if the assumption is wrong, the code is merely suboptimal, but worst-case guarantees remain intact.}. By explicitly adopting a comonotonic coding distribution, compressive refinement is guaranteed to yield more concise explanations.

Crucially, none of our formal results assume explanatory optimism to be true; a comonotonic code only makes compressibility explicitly track interpretability. When this belief fails for a given model, so does compression with a comonotonic code: if the components' semantics admit no simpler interpretation, the absence of a successful compressive refinement indicates exactly this.

\paragraph{Wiring vs.\ component complexity} To address the second tension between wiring and component complexity, ensuring that the string diagram $\syn D$ remains manageable through syntax refinement, we invoke a result from approximation theory regarding the nature of the functions neural networks can approximate. \textit{Compositional sparsity} ensures that the search for ``atoms of computation'' is feasible. The graph size $L(\syn D)$ is expected to grow only polynomially because the semantics approximated by the model admit a sparse decomposition.
\begin{theorem}[Compositional sparsity]
Deep (neural) networks of polynomial size efficiently approximate {compositionally sparse} functions. A function is \textnormal{compositionally sparse} if it can be decomposed into polynomially many constituent functions, each depending on a small, constant number of variables. This class includes all efficiently Turing-computable functions \cite{danhoferPositionTheoryDeep2025}.
\end{theorem}

Together, the operationalised principles of explanatory optimism and compositional sparsity justify minimising representation complexity as a tractable surrogate for the full description length. Compressive refinement distils the model into a sparse syntax with concise semantics that are aligned with human concepts. Therefore, compressive refinement yields better compositionally interpretable models.

\section{Related Work}
\label{sec:related_work}


\paragraph{Mechanistic interpretability}
Existing methods use various syntactic refinements $\syn R: \syn S \to \syn{S'}$ (\cref{def:syntactic_refinement}), differing in how tightly $\syn{S'}$ must track the given $\syn S$. Early work often used the identity refinement $\syn R = \syn{\textnormal{id}_S}$, treating architectural units such as neurons as the atoms of analysis \citep{olah2017feature}. Contemporary research is systematically moving to stronger decomposition classes, ranging from piecewise substitution \citep{bricken2023monosemanticity}, to relying on architectural invariants \citep{lindsey2024sparse-crosscoders}, or disregarding the original syntax altogether \citep{gaoWeightsparseTransformersHave2025}. Compositional interpretability makes this balance explicit and provides a parsimony criterion for decomposition selection. A more detailed comparison can be found in \cref{app:interpretability_comparison}.

\paragraph{Measurable interpretability} To bridge the traditional intrinsic and post-hoc interpretability paradigms, \citet{madsenInterpretabilityNeedsNew2024} propose \textit{inherently faithfulness-measurable models}. Rather than imposing inflexible constraints that directly generate explanations, such models allow efficient evaluation of explanation faithfulness by design. When a measurable target metric exists, one can search directly for metric-optimal explanations rather than rely on heuristic proxies like sparsity or low-rankness \cite{zhouSolvabilityInterpretabilityEvaluation2023}.
Compositional interpretability extends this principle to mechanistic explanations, but recognises that any metric of faithfulness must be balanced against a measure of complexity. Only when both criteria are efficiently measurable can the discovery of interpretable decompositions be cast as a standard optimisation procedure.

\paragraph{Symmetry-formalised interpretability} \citet{barbieroActionableInterpretabilityMust2026a} propose defining actionable interpretability through invariance under four symmetries. These strict constraints naturally guide the design of intrinsically interpretable models \citep{barbieroFoundationsInterpretableModels2025}, but are difficult to apply to mechanistic or post-hoc interpretations of black-box models, where such symmetries are not guaranteed to hold.
Compositional interpretability instead builds on two fundamental principles: compositionality and MDL for explanation grounding and selection, respectively. These two principles recover analogous desiderata: \textit{inference equivariance} is implied by the functoriality of compositional interpretations, \textit{information invariance} by minimisation of representation complexity via compressive refinement, \textit{concept-closure invariance} by minimisation of interpretation complexity, and \textit{structural invariance} by separation of syntax from semantics. Replacing symmetries with an MDL objective yields a spectrum of compositionally interpretable models that also includes non-interpretable models. In this spectrum, actionable interpretability is reframed from an architectural constraint into a continuous optimisation objective.

\paragraph{Causal abstraction} \citet{geigerCausalAbstractionTheoretical2025} frame mechanistic interpretability as causal alignment between low-level model decompositions and high-level explanations. While rigorous for top-down hypothesis testing, this approach assumes the high-level explanation and the decomposition search space are known a priori \cite{geigerFindingAlignmentsInterpretable2024}. Moreover, \citet{sutterNonLinearRepresentationDilemma2025} note that without assuming linear representations, this alignment becomes unidentifiable: any arbitrary high-level explanation can be mapped onto any low-level decomposition.
Compositional interpretability focuses instead on the antecedent step of structural decomposition. It instead searches bottom-up for simple decompositions from which candidate algorithms can be extracted. Its formulation is directly compatible with causal abstraction and inference \citep{lorenzCausalModelsString2023,lorenzCausalCompositionalAbstraction2026}: our definition of syntactic refinement aligns with the stronger notion of \textit{strict component-level abstraction}, and our formulation of compositional faithfulness can incorporate causal validity through the semantic distortion measure $d$. As such, compressive refinement supplies the candidate decompositions and algorithms required for distributed alignment search.

\section{Conclusion}

\paragraph{Summary} This paper establishes a formal foundation for mechanistic interpretability based on the first principles of compositionality and minimum description length. We argue that interpretability methods should explicitly aim to make models \textit{compositionally interpretable}, and that reducing an explanation's description length provides a quantifiable objective to improve its quality. We show how structural measures of \textit{faithfulness} and \textit{complexity} can serve as tractable surrogates for minimising compositional description length. 
Finally, we demonstrate that \textit{compressive refinement} generalises current mechanistic approaches to neural network decomposition and supply a parsimony criterion under which minimising representation complexity is guaranteed to produce better compositional explanations.

\paragraph{Future work} We discuss three directions to translate these theoretical foundations into practical tools for interpretation.
\begin{itemize}
\item We focus on decomposition; the description and validation stages of the reverse-engineering workflow \cite{sharkeyOpenProblemsMechanistic2025} require separate treatment, and integrating compressive refinement with causal inference \cite{joshiCausalityKeyInterpretability2026} is a natural next step.
\item Automated compressive refinement requires faithfulness and complexity to be efficiently measurable and optimisable. Tensor networks could meet these requirements, and early results suggest they can scale these objectives tractably \cite{doomsCompositionalityUnlocksDeep2024, hamrerasTensorizationPowerfulUnderexplored2025}.
\item While allowing subjective priors, standard complexity assumes unbounded computational capacity, poorly modelling humans. Epiplexity \cite{finziEntropyEpiplexityRethinking2026} proposes a bounded-observer alternative to capture extractable structure better.
\end{itemize}

\section*{Acknowledgements}

This research received funding from the Research Foundation Flanders (FWO) grant 11A8R26N (WG) and the Flanders AI Research Program (TD).


\section*{Impact Statement}

This paper presents work whose goal is to advance the field of Machine
Learning. There are many potential societal consequences of our work, none of
which we feel must be specifically highlighted here.


\bibliography{paper}
\bibliographystyle{icml2026}

\newpage
\appendix
\onecolumn
\setlist[itemize]{}

\section{String diagrams}
\label{app:string_diagram}
\textit{String diagrams} \cite{coeckePicturingQuantumProcesses2018} are a category-theoretic tool for visually describing compositional structures and formally reasoning about them. They place the composition of processes, representing the flow of information, at the centre of analysis. By linking topology to algebra, intuitive visual manipulations directly correspond to formal algebraic proofs of behavioural equivalence, with the mathematical details hidden from view. Diagrams can be manipulated via equational rules,
summarised by the principle: ``Only connectivity matters'' \cite{coeckePicturingQuantumProcesses2018}.
\begin{definition}[String Diagram]
    A \textit{string diagram} $\syn D$ is a graphical representation of a composite process (formally a morphism $\syn D: \syn I \to \syn O$ in a monoidal category $\syn S$) characterised by:
    \begin{itemize}
        \item \textit{Wires} representing types (objects $\syn A, \syn B, \syn C, \ldots \in \text{Obj}(\syn{S})$);
        \item \textit{Boxes} representing processes between types (morphisms $\syn M: \syn A \to \syn B \in \text{Hom}_{\syn{S}}(\syn A,\syn B)$).
    \end{itemize}
    Diagrams are composed via:
    \begin{itemize}
        \item \textit{Sequential composition:} Connecting output wires of one box to matching input wires of another vertically (composition of morphisms $\syn M: \syn A \to \syn B$ and $\syn N: \syn B \to \syn C$ gives $\syn N \circ \syn M: \syn A \to \syn C$);
        \item \textit{Parallel composition:} Placing boxes side-by-side horizontally (monoidal product of morphisms $\syn M: \syn A \to \syn C$ and $\syn N: \syn B \to \syn D$ gives $\syn M \otimes \syn N: \syn A \otimes \syn B \to \syn C \otimes \syn D$)
    \end{itemize}
\end{definition}
String diagrams further allow defining \textit{states} as boxes with no inputs, \textit{effects} as boxes with no outputs, and \textit{scalars} as boxes with neither inputs nor outputs. \cref{fig:diagram} gives examples of string diagrams; comprehensive introductions are available elsewhere \cite{coeckePicturingQuantumProcesses2018,shieblerCategoryTheoryMachine2021,tullCompositionalInterpretabilityXAI2024}. 

\section{Proof of the description length bound}
\label{app:proof}
\begin{proof}{Proof of \cref{prop:description_length_bound}}
By \cref{def:compositional_interpretation}
$$
L(\expl{I_S}(\syn D)) = L(\expl{I_C}(\sem{[[\,D\,]]}))
$$
The description length of a mapping's output is bounded by the sum of the description length of the mapping and its input (information non-increase). Therefore
$$
 L(\expl{I_C}(\sem{[[\,D\,]]})) \le L(\sem{[[\,D\,]]}) + L(\expl{I_C}\rvert_{\sem{{C}_{D}}})
$$
Because, by \cref{def:compositional_model}, $\sem{[[\,\cdot\,]}]$ is a structure-preserving (monoidal) functor, the description length of the semantic object $\sem{[[\,D\,]]}$ is bounded by the sum of the description lengths of its constituent parts and their arrangement (information subadditivity)
$$
L(\sem{[[\,D\,]]}) \le  L(\syn D) + \sum_{\syn f \in \text{mor}(\syn{S_D})} L(\sem{[[\,f\,]]})
$$
Combining these two inequalities yields
$$
L(\expl{I_S}(\syn D)) \le L(\syn D) + \sum_{\syn f \in \text{mor}(\syn{S_D})} L(\sem{[[\,f\,]]}) + L(\expl{I_C}\rvert_{\sem{C_{D}}}) = L^{\sem{\text{rep}}}(M) + L^{\expl{\text{int}}}(M, \expl{I_C})$$
The right-hand side matches the definition of $L(M, \expl{I_S})$ in \cref{def:compositional_description_length}, proving that $L(\expl{I_S}(\syn D)) \le L(M, \expl{I_S})$.
\end{proof}

\section{Comparison of interpretability methods} \label{app:interpretability_comparison}

Mechanistic interpretability methods trade off the scope of refinement against traceability to the original model: a broader refinement space can expose more natural atoms of computation, but at the cost of losing that trace. \cref{tab:refinement_levels} lists common methods along this spectrum, with a representative example for each category.

\begin{table*}[ht]
\centering
\small
\renewcommand{\arraystretch}{1.3}
\begin{tabularx}{\textwidth}{@{}>{\raggedright\arraybackslash}p{0.20\textwidth} >{\raggedright\arraybackslash}p{0.28\textwidth} X@{}}
\toprule
\textbf{Level} & \textbf{Constraint on $\syn R$} & \textbf{Limitation on $\expl{I_S}$} \\
\midrule

\textbf{L0.\ Post-hoc} \newline
saliency maps \newline
\citep{simonyan2014deepinsideconvolutionalnetworks} &
$\syn R$ is not constructed. &
No syntax-based interpretation $\expl{I_S}$ is created; two models
that implement the same function by different mechanisms receive the
same explanation.
\\

\textbf{L1.\ Architectural} \newline
attention heads \newline
\citep{olsson2022incontextlearninginductionheads} &
$\syn R = \syn{\text{id}_S}$. &
The atoms of analysis are fixed by the architecture. Any functional
unit that crosses architectural boundaries, or lives inside one,
cannot be captured.
\\

\textbf{L2.\ Piecewise} \newline
transcoders \newline
\citep{dunefskyTranscodersFindInterpretable2024} &
$d(\sem{[[\,f\,]]}, \sem{[[\,R(f)\,]]'}) \le \epsilon$ for each
component $\syn f$ of $\syn D$. &
Each component is refined in isolation. Regularities that span or
combine components (cross-layer superposition) are structurally
excluded --- refining them would break the per-component match.
\\

\textbf{L3.\ Global} \newline
attribution graphs \newline
\citep{lindsey2024sparse-crosscoders} &
$d(\sem{[[\,D\,]]}, \sem{[[\,R(D)\,]]'}) \le \epsilon$ only. &
Individual components of $\syn S$ need not correspond to anything in
$\syn{S'}$. Localising a claim to a specific component requires an
additional invariance in $\sem C$ (e.g.\ linearity of the residual
stream).
\\

\textbf{L4.\ Intrinsic} \newline
weight-sparse transformers \newline
\citep{gaoWeightsparseTransformersHave2025} &
$\syn R$ is not constrained. &
$\syn R$ is trivial as an interpretation of a pretrained source model; the framework analyses $M'$ directly rather than explaining a given $M$.
\\ \bottomrule
\end{tabularx}

\caption{Interpretability methods as subclasses of refinement.
Reading down column~2, the constraint on $\syn R$ weakens from
per-component match, to whole-diagram match, to trivial match
against a collapsed source; reading down column~3, what each method
cannot see grows accordingly. L4 fits only as a limiting case: the
framework analyses $M'$ directly, not as an interpretation of any
pretrained model.}
\label{tab:refinement_levels}
\end{table*}

\section{Worked example of representation complexity}
\label{app:worked_example}

To illustrate the representation complexity $L^{\sem{\text{rep}}}(M)$, we compute it for the linear animal-and-colour classifier of \autoref{fig:diagram}, encoding each parameter with $8$ bits and each box or wire with $16$ bits. The encoding is naive and not comonotonic (\cref{the:comonotonicity}), and the specific constants matter less than the direction of the inequality. The model maps $64$ input features to logits over $4$ classes (two animals, two colours), split across three wires of $8$, $24$, and $32$ features.

\paragraph{Naive decomposition} A naive decomposition treats the classifier as a single dense map $\sem{[[\,W\,]]} \in \mathbb{R}^{4 \times 64}$: one box with an input and an output wire ($3$ elements).
$$
L^{\sem{\text{rep}}} = \underbrace{3 \cdot 16}_{L(\syn D)} + \underbrace{4 \cdot 64 \cdot 8}_{\sum_{\syn f} L(\sem{[[\,f\,]]})} = 48 + 2048 = 2096 \text{ bits}
$$

\paragraph{Entangled decomposition $\syn S$} The original syntax of \autoref{fig:diagram} routes the first two input wires through a bottleneck $\sem{[[f]]} \in \mathbb{R}^{4 \times (8 + 24)}$ that feeds both an animal head $\sem{[[g]]} \in \mathbb{R}^{2 \times 4}$ and a colour head $\sem{[[h]]} \in \mathbb{R}^{2 \times (4 + 32)}$ ($10$ elements). Both heads read $\sem{[[f]]}$, so they are entangled.
$$
L^{\sem{\text{rep}}}(M) = \underbrace{10 \cdot 16}_{L(\syn D)} + \underbrace{(4 \cdot (8 + 24) + 2 \cdot 4 + 2 \cdot (4 + 32)) \cdot 8}_{\sum_{\syn f} L(\sem{[[\,f\,]]})} = 160 + 1664 = 1824 \text{ bits}
$$

\paragraph{Refined decomposition $\syn{S'}$} The compressive refinement $\syn R$ replaces $\sem{[[h]]}$ by $\sem{[[h']]} \in \mathbb{R}^{2 \times 32}$, reading only its own input wire, while $\sem{[[f]]}$ and $\sem{[[g]]}$ are unchanged ($9$ elements). This severs the wire carrying the hidden representation to the colour head, disentangling the two paths.
$$
L^{\sem{\text{rep}}}(M') = \underbrace{9 \cdot 16}_{L(\syn{D'})} + \underbrace{(4 \cdot (8 + 24) + 2 \cdot 4 + 2 \cdot 32) \cdot 8}_{\sum_{\syn f} L(\sem{[[\,f\,]]'})} = 144 + 1600 = 1744 \text{ bits}
$$

Representation complexity decreases at each step, $2096 > 1824 > 1744$, so the refinement is compressive. It also disentangles the colour head from the animal pathway, which does not raise interpretation complexity, so by the parsimony criterion (\cref{the:parsimony_criterion}) the refined model gives a more concise compositional explanation.

\section{Notation}
\label{app:notation}

The paper uses colour to distinguish three roles: $\syn{\text{syntax}}$, $\sem{\text{semantics}}$, and $\expl{\text{explanations}}$. A symbol's colour indicates the specific category an object or morphism inhabits, or the target category of a functor. Uncoloured symbols fall outside these categories.

\setlist[itemize]{}

\paragraph{Compositional models}
\begin{itemize}
    \item $M = (\syn D, \syn S, \sem C, \sem{[[\,\cdot\,]]})$ --- a compositional model (\cref{def:compositional_model}).
    \item $\syn S$ --- syntactic monoidal category; $\syn{S_D} = (\text{ob}(\syn{S_D}), \text{mor}(\syn{S_D}))$ its signature (objects and generators appearing in $\syn D$).
    \item $\sem C$ --- semantic monoidal category; $\sem{C_D}$ the image of $\syn{S_D}$ under the representation.
    \item $\syn D: \syn I \to \syn O$ --- string diagram over $\syn S$; $\syn f \in \text{mor}(\syn{S_D})$ a generator of $\syn D$.
    \item $\sem{[[\,\cdot\,]]}: \syn S \to \sem C$ --- monoidal representation functor; $\sem{[[\,f\,]]}$ the semantic instantiation of the generator $\syn f$; $\sem{[[\,D\,]]}$ the semantics of the full diagram.
    \item $M' = (\syn{D'}, \syn{S'}, \sem C, \sem{[[\,\cdot\,]]'})$ --- a refined compositional model sharing semantics $\sem C$ with $M$ (\cref{def:syntactic_refinement}).
\end{itemize}

\paragraph{Interpretations}
\begin{itemize}
    \item $\expl H$ --- signature of human-interpretable terms.
    \item $\expl{I_S}: \syn S \to \expl H$ --- compositional interpretation (syntactic grounding).
    \item $\expl{I_C}: \sem C \to \expl H$ --- post-hoc interpretation (semantic grounding); $\expl{I_C}\rvert_{\sem{C_D}}$ its restriction to the semantic components appearing in $\syn D$.
    \item $\sem{I^*_C}: \expl H \to \sem C$ --- reconstruction map witnessing compositional faithfulness (\cref{def:compositional_faithfulness}).
    \item Commutativity $\expl{I_S} = \expl{I_C} \circ \sem{[[\,\cdot\,]]}$ links the two groundings.
\end{itemize}

\paragraph{Refinements}
\begin{itemize}
    \item $\syn R: \syn S \to \syn{S'}$ --- syntactic refinement; a functor with $\syn{D'} = \syn R(\syn D)$ and $\sem{[[\,\cdot\,]]} = \sem{[[\,\cdot\,]]'} \circ \syn R$.
    \item $\syn R = \syn{\text{id}_S}$ --- the identity refinement (L1 in \cref{tab:refinement_levels}).
    \item A refinement is \textit{compressive} when $L^{\sem{\text{rep}}}(M') \le L^{\sem{\text{rep}}}(M)$.
\end{itemize}

\paragraph{Description lengths}
\begin{itemize}
    \item $L(\,\cdot\,)$ --- description length (in bits).
    \item $L^{\sem{\text{rep}}}(M) = L(\syn D) + \sum_{\syn f \in \text{mor}(\syn{S_D})} L(\sem{[[\,f\,]]})$ --- representation complexity; upper-bounds $L(\sem{[[\,D\,]]})$ (\cref{app:proof}).
    \item $L^{\expl{\text{int}}}(M, \expl{I_C}) = L(\expl{I_C}\rvert_{\sem{C_D}})$ --- interpretation complexity.
    \item $L(M, \expl{I_S}) = L^{\sem{\text{rep}}}(M) + L^{\expl{\text{int}}}(M, \expl{I_C})$ --- compositional description length; upper-bounds $L(\expl{I_S}(\syn D))$ (\cref{prop:description_length_bound}).
\end{itemize}

\paragraph{Faithfulness}
\begin{itemize}
    \item $d$ --- semantic distortion metric on parallel morphisms in $\sem C$ (or a generalisation thereof).
    \item $\epsilon \in \mathbb{R}$ --- tolerance for relaxed equalities.
    \item $d(\sem{[[\,D\,]]}, \sem{[[\,R(D)\,]]'}) \le \epsilon$ --- behavioural faithfulness of a refinement $\syn R$.
    \item $d(\sem{[[\,f\,]]}, \sem{I^*_C}(\expl{I_S}(\syn f))) \le \epsilon$ for every $\syn f \in \text{mor}(\syn S)$ --- compositional faithfulness of an explanation (\cref{def:compositional_faithfulness}).
\end{itemize}


\end{document}